\documentclass{article}
\usepackage{amsmath,amssymb,amsthm} 
\usepackage{spconf,amsmath,graphicx}
\usepackage[tight,normalsize]{subfigure}
\usepackage[ruled,linesnumbered,titlenumbered]{algorithm2e}

\newtheorem{proposition}{Proposition}

\newtheorem{lemma}{Lemma}
\DeclareMathSizes{9}{9}{6.3}{4.5}

\def\L{{\cal L}}

\title{Efficient Superimposition Recovering Algorithm}
\name{Han Li$^{\star}$\qquad Kun Gai$^{\star \dagger}$ \qquad Pinghua Gong$^{\star}$ \qquad Changshui Zhang$^{\star}$ }
\address{$^{\star}$ Dept. of Automation, Tsinghua University, Beijing, China\\
$^{\dagger}$ Alibaba Group, Beijing, China}
\begin{document}
%
\maketitle

\def\etal{et al.}
\def\eg{e.g.}

\begin{abstract}
In this article, we address the issue of recovering latent transparent layers from superimposition images. Here, we assume we have the estimated transformations and extracted gradients of latent layers. To rapidly recover high-quality image layers, we propose an Efficient Superimposition Recovering Algorithm (ESRA) by extending the framework of accelerated gradient method. In addition, a key building block (in each iteration) in our proposed method is the proximal operator calculating. Here we propose to employ a dual approach and present our Parallel Algorithm with Constrained Total Variation (PACTV) method. Our recovering method not only reconstructs high-quality layers without color-bias problem, but also theoretically guarantees good convergence performance.
\end{abstract}

\begin{keywords}
Superimposition recovering, proximal operator, optimization
\end{keywords}

\vspace{30pt}


Assume we have the extracted gradients and transformations of each latent transparent layer, the reconstruction step is the final crucial part for reflection separation. We assume the variations of transmitted layers in each \textit{mixtures} conform to a parametric transformation $f(x,\theta)$ ($x$ is the pixel coordinates) with different parameters $\theta_i$. Here we propose an Efficient Superimposition Recovering Algorithm (ESRA) to fast recover the high quality latent layers.

\section{Efficient Superimposition Recovering Algorithm}\label{sec:Restoration:ssec:RecoverObj}

\def\vec{\text{vec}}
\def\lv{l^{vec}}

With estimated transformation parameters $\theta_i$, we align the transmitted layers by warping $\textit{mixtures}$ $I_i$ with $f^{-1}(x,\theta_i)$. Then our mixing model is rewritten as:
\begin{equation}
I_i(f_i^{-1}(\!x\!)) = a_{i1}L^t(\!x\!)+a_{i2}L^{r(\!i\!)}(f_i^{-1}(\!x\!)),\qquad i=1,\cdots,m.
\label{eq:warpedmixingmodel}
\end{equation}
Here $L^t$ is the latent transmitted layer, $L^{r(\!i\!)}$ is the reflected layer in $i$th \textit(mixtures), $a_{i1}, a_{i2}$ is the mixing coefficients. With this new mixing model, the influence of parametric transformations $f(x,\theta_i)$ can be ignored in the intermediate recovering process. For simplicity, we use $I_i(\!x\!)$ to represent $I_i(f_i^{-1}(\!x\!))$. $L_1(x)$ and $L_{i+1}(\!x\!)$ denote $L^t(\!x\!)$ and $a_{2i}L^{r(\!i\!)}(f_i^{-1}(\!x\!))$, respectively. Let $E^i(\!x\!)$ stand for the extracted gradients from $L^i(\!x\!)$. To recover high quality latent image layers, we propose to employ $L_1$ penalty on the extracted gradients and nonnegative constraints on the layers' intensities along with the $L_2$ loss of the mixing model. Thus our recovering objective function is written as:
\vspace{-8pt}
\begin{equation}
\begin{aligned}
&\!\!\!\!\!\underset{0\leq \lv \leq1}{\min{}}F(\lv)\! =\!\lambda{}\sum_{x,i=1}^{m+1}| \nabla L^i(\!x\!)\!
-\! E^{i}(\!x\!)| \\ &+
\!\!\!\sum_{x,i=1}^{m}\!\frac{1}{2}\big(I_{i}(\!x\!)\!-\!a_{i1}L_1(\!x\!)\!-\!L_{i\!+\!1}(\!x\!)\big)\!^2
\label{eq:originalLoss}
\end{aligned}
\vspace{-5pt}
\end{equation}
where $\lv \doteq{}[\vec^{\top}(L^1),\cdots,\vec^{\top}(L^{m\!+\!1})]^\top$ is a large vector containing all pixel values in all latent layers. The first $L_1$ term enforces the agreement between reconstructed layer gradients and extracted layer gradients, while the second $L_2$ term tends to satisfy our mixing mode. Since the extracted gradients are nonzero at very few coordinates, the $L_1$ norm term not only prefers layers with sparse gradients but also avoids over-smooth results. $\lambda$ is a trade off coefficient.


\def\arg{\text{argmin}}
\def\L{L_s}

To solve the nonsmooth convex optimization model (\ref{eq:originalLoss}) efficiently, we denote
\vspace{-5pt}
\begin{equation}
\begin{aligned}
f(\lv)&=\sum_{x,i=1}^{m}\!\frac{1}{2}\big(I_{i}(\!x\!)\!-\!a_{i1}L_1(\!x\!)\!-\!L_{i\!+\!1}(\!x\!)\big)\!^2, \  \text{s.t}\ 0\leq \lv \leq 1,\\
g(\lv)&= \!\lambda{}\sum_{x,i=1}^{m+1}| \nabla L^i(\!x\!)
- E^{i}(\!x\!)|.
\label{eq:GFX}
\end{aligned}
\vspace{-2pt}
\end{equation}
Here $g(\lv)$ is the $\ell_1$ penalty on the extracted gradients and $f(\lv)$ corresponds to the $L_2$ loss and nonnegative constraints. $f(\lv)$ can be formulated in the following matrix form:
\vspace{-5pt}
\begin{equation}
\begin{aligned}
f(\lv)&= \frac{1}{2}||A\lv-b||^2,\ \ \mathrm{s.t.}~0\leq{}\lv\leq{}1,\\
\text{where~}
A&=\left[
\begin{array}{@{}c@{\:}c@{\:}c@{\:}c@{}}
a_{11}I& I&      & \\
\vdots &  &\ddots& \\
a_{m1}I&  &      &I
\end{array}
\right],
b = \left[
\begin{array}{@{}c}
\vec(I_1)\\
\vdots  \\
\vec(I_m)
\end{array}
\right],
\label{eq:FX}
\end{aligned}
\vspace{-5pt}
\end{equation}
where $f(\lv)$ is continuously differentiable and $\nabla{}f(\lv)=A^\top(A\lv\!-\!b)$, of which Lipschitz constant $L(f)\!=\!\lambda_{max}(A^\top{}\!A)\!=\!\sum_ia_{i1}^2\!+\!1$, and $I\!\in\!\mathbb{R}^{hw\!\times\!hw}$ is the unit matrix.
We note the objective function in (\ref{eq:originalLoss}) is a composite function of a differential term $f(\lv)$ and a non-differential term $g(\lv)$. Denote
\vspace{-5pt}
\begin{equation}
\begin{aligned}
&P_{\L,\lv_{k-1}}(\lv) = f(\lv_{k-1}) + \langle \nabla f(\lv_{k-1}),\lv - \lv_{k-1} \rangle \\ &+ \frac{\L}{2}\|\lv - \lv_{k-1}\|^2,
\label{eq:proximalMap}
\vspace{-5pt}
\end{aligned}
\end{equation}
which is the first order Taylor expansion of $f(\lv)$ at $\lv_{k-1}$, with the squared Euclidean distance between $\lv$ and $\lv_{k-1}$ as the regularization term. The traditional gradient descent algorithm obtains the solution at the $k$-th iteration $(k\geq 1)$ by $\lv_k =  \text{arg} \min P_{\L,\lv_{k-1}}(\lv) + g(\lv)$  with a proper step size $\L$ (greater than $L(f)$). Here we propose to employ the accelerated gradient descent \cite{nemirovski2005efficient,nesterov2004introductory} to solve the reconstruction problem, named Efficient Superimposition Recovering Algorithm (ESRA). Here we generate a solution at the $k$-th iteration $(k\geq 1)$ by computing the following proximal operator
\vspace{-5pt}
\begin{equation}
\lv_k \to \text{arg}\min_{ 0\leq \lv \leq 1} P_{\L,Y_k}(\lv) + g(\lv)
\label{eq:proximalOperator}
\vspace{-5pt}
\end{equation}
where $Y_1 = \lv_0$ and $Y_k = \lv_{k-1} + \frac{t_{k-2}-1}{t_{k-1}}(\lv_{k-1} - \lv_{k-2})$ for $k\geq 1$. We note that $Y_k$ is a linear combination of $\lv_{k-1}$ and $\lv_{k-2}$. The combination coefficient plays an important role in the convergence of the algorithm. As suggested by \cite{beck2009}, we set $t_0 =1 $ and $t_k = (1+ \sqrt{t_{k-1}^2 +1})/2 $ for $k\geq 1 $. According to the theoretical analysis in \cite{beck2009}, this accelerated gradient descent method can get within $O(1/k^2)$ of the optimal objective value after $k$ steps. While solving problem (\ref{eq:proximalOperator}) is still very challenging, we propose a Parallel Algorithm with Constrained Total Variation (PACTV) method to find the optimal solution, which is presented in the sequel.

\def\lv{l^{vec}}
\def\arg{\text{argmin}}
\begin{algorithm}[bht]\label{alg:ESRA}
\SetKwInOut{Input}{Input} \SetKwInOut{Output}{Output}
\Input{$\L\geq L(f)$- An upper bound on the Lipschitz constant of $\nabla f$. $N$ is the total number of iterations.}
\textbf{Step 0}: Take $Y_1=\lv_0=0_{(hw(m+1),1)},t_1=1.$\\

\textbf{Step k}:$(k\geq1)$ Compute
\vspace{-5pt}
\begin{eqnarray}
\lv_k&=&\text{arg}\min_{0\leq \lv \leq 1}P_{\L,Y_k}(\lv) + g(\lv),\qquad\qquad\\
\vspace{-3pt}
t_{k}&=&\frac{1+\sqrt{1+4t_{k-1}^2}}{2},\qquad\qquad\qquad\\
\vspace{-3pt}
Y_{k+1}& = &\lv_{k} + \frac{t_{k-1}-1}{t_k}(\lv_k -\lv_{k-1})
\end{eqnarray}

\Output{$\lv_N$ is the final recovered result.}
\caption{ESRA($\L,N$)}
\end{algorithm}

\section{PACTV via dual approach}\label{sec:Restoration:ssec:FindPLYk}

Given problem (\ref{eq:proximalOperator}), we observe it can be solved block separable in the following way. If we denote $Y_k\!-\!\frac{1}{\L}\nabla{}f(\!Y_k\!)\doteq[\vec^\top\!(d_1),\cdots,\vec^\top\!(d_{m\!+\!1})]^\top$ ($d_i\in \mathbb{R}^{h\!\times{}\!w}\ i\!=\!1,\!\cdots\!,m\!+\!1)$, we can split $Y_k \!-\!\frac{1}{\L}\nabla{}f(\!Y_k\!)$ into $m+1$ separable parts. Then by employing the definition of (\ref{eq:GFX}), we transform (\ref{eq:proximalOperator}) into the following form:
\vspace{-8pt}
\begin{equation}
\begin{aligned}
\lv_k =& \underset{0\leq \lv_n \leq 1}{\arg}\big\{\sum_{i=1}^{m+1}\sum_{x}\big(\lambda{}|\nabla L^i(x)\!
-\!E^{i}(x)|\!\\
&+\!\frac{\L}{2}||L^i(x)\!-\!d_i(x)||^2\big)\big\}.
\label{eq:PLYkobj}
\end{aligned}
\vspace{-5pt}
\end{equation}
As illustrated in (\ref{eq:PLYkobj}), finding $\lv_k$ is to solve following $m\!+\!1$ separable problems with constrained total variation in parallel:
\vspace{-5pt}
\begin{equation}
\begin{aligned}
\min_{0\leq L \leq 1} \sum_x \big(\!\frac{1}{2}||L(x)\!-\!d(x)||^2&+\beta|\nabla L(x)\!-\!E(x)|\big).
\label{eq:FGPobjective}
\end{aligned}
\vspace{-5pt}
\end{equation}
Here $\beta=\lambda/\L$, and $L,d,E$ represent $L^i,d_i,E^i$, respectively. Similar with the image denoising problem~\cite{beck2009fast,beck2009}, we propose a dual approach to solve (\ref{eq:FGPobjective}) and give some notation in order:
\vspace{-5pt}
\begin{itemize}
\item {$\mathcal{P}$ is the set of matrix-pairs $(p,q)$ where $p\in \mathbb{R}^{(h\!-\!1)\!\times{}\!w}$ and $q\in \mathbb{R}^{h\!\times{}\!(w\!-\!1)}$ that satisfy $|p_{i,j}|\leq 1~and~|q_{i,j}|\leq 1\quad \forall i,j$. And we assume $p_{0,j}=p_{h,j}=q_{i,0}=q_{i,w}\equiv 0$, for every $i=1,\cdots,h,~j=1,\cdots,w$.
}
\item {The linear operation $\mathcal{L}:\mathbb{R}^{(h\!-\!1)\!\times{}\!w}\times{}\mathbb{R}^{h\!\times{}\!(w\!-\!1)} \to \mathbb{R}^{h\!\times{}\!w}$ is defined by the formula $\mathcal{L}(p,q)_{i,j}=p_{i\!-\!1,j}+q_{i,j\!-\!1}-p_{i,j}+q_{i,j} \quad \forall i,j.$
    }
\item {The operator $\mathcal{L}^T:\mathbb{R}^{h\!\times{}\!w} \to \mathbb{R}^{
    (h\!-\!1)\!\times{}\!w}\times{}\mathbb{R}^{h\!\times{}\!(w\!-\!1)} $ which is adjoint to $\mathcal{L}$ is given by $\mathcal{L}^\top (L)=(p,q),$ where $p_{i,j}=L_{i\!+\!1,j}-L_{i,j}$ and $q_{i,j}=L_{i,j\!+\!1}-L_{i,j}$.
    }
\item
{ $P_C$ is the orthogonal projection operator on the convex closed set $C = \{L : 0\leq L \leq 1\}$.}
\end{itemize}
Equipped with these notation, we derive a dual problem of (\ref{eq:FGPobjective}), and give following proposition to state the relation between the primal and dual optimal solutions.
\begin{proposition}\label{theorem:dualfunction}
\vspace{-5pt}
Let $(p,q)\in \mathcal{P}$ be the optimal solution of the problem
\vspace{-5pt}
\begin{equation}
\begin{aligned}
\underset{(p,q)\in \mathcal{P}}{\min}\big\{ H(p,q)\equiv\frac{1}{2}\!(\!-\!||H_C(\!d\!-\!\beta\!\mathcal{L}(p,q))||^2\!+\!\\||d\!-\!\beta \mathcal{L}(p,q)||^2)+\beta\big[Tr(p^\top\!E_1)+Tr(q^\top\!E_2)\big] \big\}.\label{eq:dualFGPpq}
\end{aligned}
\vspace{-5pt}
\end{equation}
where $H_C(L)=L-P_C(L)$ for every $L\in \mathbb{R}^{h\!\times{}\!w}$. Then the optimal solution of (\ref{eq:FGPobjective}) is given by $L = P_C(d-\beta \mathcal{L}(p,q))$.
\end{proposition}

\begin{proof} First note the following relation holds true:
\begin{equation}
|x| = \max_{p}\{px:|p|\leq 1\}.
\label{eq:absinequality}
\end{equation}
Hence, we can give
\begin{equation}
\begin{aligned}
\sum_k|\nabla_kL-E_k| = \max_{(p,q)\in \mathcal{P}} T(L,p,q),
\label{eq:dualdefine}
\end{aligned}
\end{equation}
where,
\begin{equation}
\begin{aligned}
T(L,p,q) = &\sum_{i=1}^{h\!-\!1}\sum_{j=1}^{w\!-\!1}\big[ p_{i,j}(L_{i\!+\!1,j}-L_{i,j}-E_{1_{i,j}})\\
&+ q_{i,j}(L_{i,j\!+\!1}-L_{i,j}-E_{2_{i,j}})
\big]\\
&+\sum_{i=1}^{h\!-\!1}p_{i,w}(L_{i\!+\!1,w}-L_{i,w}-E_{1_{i,w}})\\
&+\sum_{j=1}^{w\!-\!1}p_{h,j}(L_{h,j\!+\!1}-L_{h,j}-E_{2_{h,j}}).
\label{eq:TLpq}
\end{aligned}
\end{equation}
With this notation we have
\begin{equation}
\begin{aligned}
T(L,p,q) = Tr(\mathcal{L}(p,q)^\top L)-Tr(p^\top E_1)-Tr(q^\top E_2).
\label{eq:TLpqTrnotation}
\end{aligned}
\end{equation}
Thus the original problem (\ref{eq:FGPobjective}) becomes
\begin{equation}
\begin{aligned}
\min_{0\leq L \leq 1} \max_{(p,q)\in \mathcal{P}} \Big\{&\frac{1}{2}\|L-d\|^2 +
  \beta \big[Tr(\mathcal{L}(p,q)^\top L)\\&-Tr(p^\top E_1)-Tr(q^\top E_2)\big] \Big\}.
\label{eq:dualproblemform}
\end{aligned}
\end{equation}
Since the objective function is convex in $L$ and concave in $p,q$, we can exchange the order of the minimum and maximum and get
\begin{equation}
\begin{aligned}
\max_{(p,q)\in \mathcal{P}} \min_{0\leq L \leq 1} \Big\{&\frac{1}{2}\|L-d\|^2 +
  \beta \big[Tr(\mathcal{L}(p,q)^\top L)\\&-Tr(p^\top E_1)-Tr(q^\top E_2)\big] \Big\}.
\label{eq:mmdualproblemform}
\end{aligned}
\end{equation}
and which can be written as
\begin{equation}
\begin{aligned}
\max_{(p,q)\in \mathcal{P}} \min_{0\leq L \leq 1} \Big\{&\frac{1}{2}\big[\|L-(d-\beta \mathcal{L}(p,q))\|^2 -\|d-\beta \mathcal{L}(p,q)\|^2
 \\&+ \|d\|^2\big]-\beta \big[Tr(p^\top E_1)+Tr(q^\top E_2)\big] \Big\}.
\label{eq:mmdualproblemformsimple}
\end{aligned}
\end{equation}
Thus the optimal solution of the inner minimization problem is
\begin{equation}
\begin{aligned}
L = P_{\{0\leq L \leq 1\}}(d-\beta \mathcal{L}(p,q)).
\label{eq:finalsolutionL}
\end{aligned}
\end{equation}
And last, we plug the above expression for $L$ back into (\ref{eq:mmdualproblemformsimple}) and ignore the constant term, we obtain the dual problem
is
\begin{equation}
\begin{aligned}
\underset{(p,q)\in \mathcal{P}}{min}\big\{ H(p,q)\equiv \frac{1}{2}( - ||H_C(d-\beta \mathcal{L}(p,q))||^2+\\||d-\beta \mathcal{L}(p,q)||^2) + \beta \big[Tr(p^\top E_1)+Tr(q^\top E_2)\big] \big\},\nonumber
\end{aligned}
\end{equation}
which is the same as (\ref{eq:dualFGPpq}).
\end{proof}

what's more, given (\ref{eq:dualFGPpq}), we can easily have following lemma.
\vspace{-5pt}
\begin{lemma}\label{lemma:derOfh}
The objective funtion $H$ of (\ref{eq:dualFGPpq}) is continuously differentiable and its gradient is given by
\vspace{-5pt}
\begin{eqnarray}
\begin{aligned}
\nabla{}H(p,q)=-\beta\mathcal{L}^\top P_C(d-\beta\mathcal{L}(p,q))+\beta(E_1,E_2).\label{eq:derOfH}
\end{aligned}
\end{eqnarray}
And let $L(H)$ be the Lipschitz constant of $\nabla{}H(p,q)$, then $L(H)\leq 8\beta^2$.
\end{lemma}
\begin{proof} Consider the function $s:\mathbb{R}^{h\!\times\!w}\rightarrow \mathbb{R}$ defined by
\begin{equation}
\begin{aligned}
s(L)=\|H_C(L)\|^2.\label{eq:defineofs}
\end{aligned}
\end{equation}
Then the dual function (\ref{eq:dualFGPpq}) can be written as:
\begin{equation}
\begin{aligned}
H(p,q)=\frac{1}{2}( - s(d-\beta \mathcal{L}(p,q))+||d-\beta \mathcal{L}(p,q)||^2\big)\\ + \beta \big[Tr(p^\top E_1)+Tr(q^\top E_2)\big].\label{eq:defineofs}
\end{aligned}
\end{equation}
Obviously, $s(\cdot)$ is continuously differentiable and its gradient is given by
\begin{equation}
\begin{aligned}
\nabla s(L)=2(L-P_C(L)).
\end{aligned}
\end{equation}
Therefore,
\begin{equation}
\begin{aligned}
&\nabla H(p,q) \\
&= \frac{1}{2}\nabla \big(- s(d-\beta \mathcal{L}(p,q))+||d-\beta \mathcal{L}(p,q)||^2 \big)+\beta(E_1,E_2)\\
&=\frac{1}{2}\beta\mathcal{L}^\top \big(\nabla s(d-\beta \mathcal{L}(p,q))-2(d-\beta \mathcal{L}(p,q)) \big)+\beta(E_1,E_2)\\
&=-\beta\mathcal{L}^\top P_C(d-\beta \mathcal{L}(p,q))+\beta(E_1,E_2)
\end{aligned}
\end{equation}

Then for every two pairs of matrices $(p_1,q_1),(p_2,q_2)$ where $p_i\in \mathbb{R}^{(h\!-\!1)\!\times\!w}$ and $q_i\in \mathbb{R}^{w\!\times\!(w\!-\!1)}$ for $i=1,2$, we have
\begin{equation}
\begin{aligned}
&\|\nabla H(p_1,q_1)- \nabla H(p_2,q_2)\|\\
&= \beta\|\mathcal{L}^\top [P_C(d-\beta \mathcal{L}(p_1,q_1))]- \mathcal{L}^\top [P_C(d-\beta \mathcal{L}(p_2,q_2))]\|\\
&\leq \beta\| \mathcal{L}^\top \| \|P_C(d-\beta \mathcal{L}(p_1,q_1))-P_C(d-\beta \mathcal{L}(p_2,q_2)) \|\\
&\leq \beta^2 \| \mathcal{L}^\top \| \|\mathcal{L}(p_1,q_1)-\mathcal{L}(p_2,q_2) \|\\
&\leq \beta^2 \| \mathcal{L}^\top \| \|\mathcal{L}\| \|(p_1,q_1)-(p_2,q_2) \|\\
&= \beta^2 \| \mathcal{L}^\top \|^2 \|(p_1,q_1)-(p_2,q_2) \|
\end{aligned}
\end{equation}
here the above inequalities follow from the non-expensiveness property of the orthogonal projection operator and property of linear operators $\mathcal{L},\mathcal{L}^\top$. And from \cite{beck2009fast}, we have $\|\mathcal{L}^\top(x) \|\leq \sqrt{8}\|x\|$. Therefore, implying that $\|\mathcal{L}^\top\| \leq \sqrt{8}$ and hence $L(H)\leq 8\beta^2$.
\end{proof}
With definition of $H(p,q)$ and $\nabla H(p,q)$, fast gradient projection (FGP) is applied on the dual problem (\ref{eq:dualFGPpq}). And the complexity of each iteration in FGP is $O(hw)$. Above all, our proposed Parallel Algorithm with Constrained Total Variation (PACTV) is using FGP to solve the $m+1$ dual problems (\ref{eq:dualFGPpq}) in parallel. Then we catenate the optimal $L_i^*$ $(i=1,\dots,m\!+\!1)$ and resize them into vector form to achieve $\lv_k$.

\begin{algorithm}[bht]\label{alg:FGP}
\SetKwInOut{Input}{Input} \SetKwInOut{Output}{Output}
\Input{$d\in \mathbb{R}^{h\!\times\!w}$, $\beta=\lambda/\L$, $N$ is the total number of iterations, $E_1,E_2$.}
\textbf{Step 0}: Take $(\tilde{p}_1,\tilde{q}_1)=(p_0,q_0)=(0_{(h\!-\!1)\!\times\!w},0_{h\!\times\!(w\!-\!1)}),t_1=1.$\\
\textbf{Step k}:$(n\geq1)$ Compute
\vspace{-5pt}
\begin{eqnarray}
(p_k,q_k)&=&P_{\mathcal{P}}\big[(\tilde{p}_k,\tilde{q}_k)-\frac{1}{8\beta^2}\nabla{}H(\tilde{p}_k,\tilde{q}_k) \big],\qquad\qquad\\
\vspace{-3pt}
t_{k+1}&=&\frac{1+\sqrt{1+4t_n^2}}{2},\qquad\qquad\qquad\\
\vspace{-3pt}
(\tilde{p}_{k\!+\!1},\tilde{q}_{k\!+\!1})\!&=&\!(p_k,q_k)+(\frac{t_k\!-\!1}{t_{k\!+\!1}})(p_k\!-\!p_{k\!-\!1},q_k\!-\!q_{k\!-\!1})
\vspace{-5pt}
\end{eqnarray}

\Output{$L^*$ An optimal solution of (\ref{eq:FGPobjective}) up to a tolerance.}
\caption{FGP($b,\beta,N,E_1,E_2$)}
\end{algorithm}


Given above proposition and lemma, we can use the fast gradient projection (FGP) on dual problem (\ref{eq:dualFGPpq}). Fast gradient projection (FGP) is outlined in Algorithm~\ref{alg:FGP}. Here $P_{\mathcal{P}}(p,q)$ means projecting the matrix-pair $(p,q)$ on the set $\mathcal{P}$.  And finally we achieve the optimal solution of (\ref{eq:FGPobjective}). Then our recovering method ESRA is outlined in Algorithm~\ref{alg:ESRA}.

In our implementations, we set the total iteration number of ESRA is 100 and FGP tolerance is $0.0001$, and we also set $\L=2L(f)$ to ensure a constant stepsize. The initial value of $\lv$ is zero. The final recovered reflected layers of (\ref{eq:originalLoss}) should be warped with $f_i$ and enhance the intensity by 2 to be visible. Our recovering method launches a general optimization framework and can be extended to solve other reconstruction problems in~\cite{GaiKun09CVPR,gai2011blind}.

\bibliographystyle{IEEE}
\bibliography{references2}

\end{document}